\documentclass[a4paper,UKenglish,cleveref, autoref, thm-restate]{lipics-v2021}

\pdfoutput=1 
\hideLIPIcs  

\title{Online Algorithms with Unreliable Guidance}

\author{Julien Dallot}{TU Berlin, Germany}{}{}{}
\author{Yuval Emek}{Technion, Israel}{}{}{}
\author{Yuval Gil}{Reykjavik University, Iceland}{}{}{}
\author{Maciej Pacut}{TU Berlin, Germany}{}{}{}
\author{Stefan Schmid}{TU Berlin and Weizenbaum Institute, Germany}{}{}{}

\usepackage{stmaryrd}
\usepackage{microtype}
\usepackage{algorithm}
\usepackage{algorithmic}
\usepackage{graphicx}
\usepackage{subcaption}
\usepackage{booktabs} 

\usepackage{hyperref}
\definecolor{blueLink}{rgb}{0,0.2,0.8}
\hypersetup{colorlinks,linkcolor=blueLink,urlcolor=blueLink,citecolor=blueLink}

\usepackage{amsmath}
\usepackage{amssymb}
\usepackage{mathtools}
\usepackage{amsthm}

\theoremstyle{plain}

\newcommand{\WHEN}[1]{%
  \STATE \textbf{when} #1 \textbf{do:} \begin{ALC@g}
}
\newcommand{\ENDWHEN}{%
  \end{ALC@g}
}

\usepackage{graphicx, float} 

\usepackage{amsmath,amsthm,amssymb,thmtools,thm-restate,booktabs,graphicx,latexsym,xcolor,xspace}

\usepackage{tikz}
\usetikzlibrary{decorations.pathreplacing}

\newcommand{\bipalg}{\textrm{Ranking-DTB}}
\newcommand{\markalg}{\textrm{Marking-DTB}}
\newcommand{\mtsalg}{\textrm{MTS-DTB}}

\providecommand{\P}{}\renewcommand{\P}{\mathcal{P}}
\providecommand{\R}{}\renewcommand{\R}{\mathcal{R}}
\providecommand{\A}{}\renewcommand{\A}{\mathcal{A}}
\providecommand{\G}{}\renewcommand{\G}{\mathcal{G}}
\newcommand{\Alg}{\mathtt{Alg}}
\newcommand{\Opt}{\mathtt{Opt}}
\newcommand{\Reals}{\mathbb{R}}
\providecommand{\Ex}{}\renewcommand{\Ex}{\mathbb{E}}
\newcommand{\Support}{\operatorname{support}}

\authorrunning{J. Dallot, Y. Emek, Y. Gil, M. Pacut, S. Schmid} 



\funding{
  The research of J.~Dallot and S.~Schmid is supported by the German Research Foundation (DFG), Schwerpunktprogramm SPP 35 2378, ReNO-2 (511099228), 2025-2029.
  The research of Y.~Emek is partially supported by the Israel Science Foundation (ISF), grant 730/24.
}
\nolinenumbers 

\begin{document}
\maketitle

\begin{abstract}
  This paper introduces \emph{online algorithms with unreliable guidance (OAG)}, a model for ML-augmented online decision-making that cleanly separates the predictive and algorithmic components, thus offering a single, well-defined analysis framework that depends only on the problem at hand.
  Formulated through the lens of request-answer games, the OAG model brings multiple concepts (predictions from the answer space, guide, anytime competitiveness) which enable learning-augmented algorithms to be analyzed independently of predictor-specific choices---such as prediction semantics, error functions, or probing strategies---that would otherwise restrict the algorithm's generality and applicability.
  The clean framework of the OAG model allows to build the first generic compiler, the \emph{drop-or-trust-blindly (DTB) compiler}, that turns almost any standard, prediction-free online algorithm into a learning-augmented one.
  Although simple, we show that the DTB compiler produces new learning-augmented algorithms with strong consistency-robustness guarantees for three classic online problems:
  we achieve new trade-offs for \emph{bipartite matching} with adversarial arrival order, and obtain optimal solutions for \emph{caching} and \emph{uniform metrical task systems}.
\end{abstract}

\section{Introduction}

Introduced in the influential ICML 2018 paper of Lykouris and Vassilvitskii~\cite{lykouris2021competitive}, Learning-Augmented online algorithms refer to online algorithms which accept an additional argument, the \emph{prediction}; the goal is to outperform traditional online algorithms in case of accurate predictions and still provide provable performance guarantees regardless of prediction accuracy.
At their core, learning-augmented algorithms seek for the proper use of ``guessing machines'' --- typically an ML predictor --- to solve concrete problems, especially when the predictor is a complete black-box and the prediction is therefore unreliable by nature.
Since their introduction, this framework gathered significant attention and was applied to multiple problems --- the literature on learning-augmnented online algorithms is too extensive to list here;
see \cite{OnlineAlgorithmsPredictionsWebsite} for a comprehensive record.

Over the years, multiple models have been proposed to describe and analyze learning-augmented algorithms.
In essence, those models seek to optimize the same three goals:
(1)
\emph{consistency}, the competitive ratio when the predictions are perfect;
(2)
\emph{robustness}, the competitive ratio when the predictions are arbitrarily bad;
and
(3)
\emph{smoothness}, the function describing the decline in the competitive ratio as predictions' quality degrades.
While the community agrees on these three metrics (\emph{what to measure}), their concrete implementations (\emph{how to measure}) remain debated and differ across models.
The key question is: \emph{how can one measure the accuracy of a prediction while treating the predictor as a black box?}
We review below how existing models interact with the predictor and the issues this raises.

\subparagraph*{Prediction semantics.}
Before designing a learning-augmented algorithm, one must fix a \emph{prediction semantic} --- that is, decide what a prediction represents and how many distinct predictions are possible.
This choice is consequential: it can affect the power of a learning-augmented algorithm, as shown in~\cite{DBLP:journals/iandc/Angelopoulos23} for the line search problem, where three different prediction semantics yield different optimal results.
Yet no standard way to choose a prediction semantic exists, and the choice is left to the analyst's judgment.
As a result, semantics vary across problems and sometimes even across algorithms for the same problem (e.g., predicting nodes' degrees~\cite{DBLP:conf/nips/AamandCI22} or the matched node~\cite{DBLP:conf/icml/ChooGL024} for online bipartite matching, predicting the cache content~\cite{antoniadis2023online} or the page to evict~\cite{lykouris2021competitive} for caching).
This inconsistency makes results for the same problem incomparable.

\subparagraph*{Error functions.}
Most existing learning-augmented online algorithms are analyzed using an \emph{error function}~\cite{lykouris2021competitive}, which acts as an intermediary metric to assess the quality of a prediction given a problem instance.
Initially motivated by the loss function optimized during ML training, the error function is now widely used for analysis purposes, especially to derive crucial \emph{smoothness} results, i.e., the competitive ratio for specific error values.

As of today, no consensus exists on how to choose an error function for a given problem; in practice, the choice of how to assess errors is largely left to the analyzer's judgment and justified by intuitive considerations.
Consequently, a wide variety of error functions appear in the literature: some compare predictions against aspects of the problem instance~\cite{DBLP:conf/podc/Ben-DavidDEG25,DBLP:conf/nips/PurohitSK18}, others against an optimal offline algorithm~\cite{DBLP:conf/nips/AamandCI22,bansal2022learning}, and different functions are sometimes used for similar problems (e.g., for bipartite matching, sanctioning any error~\cite{DBLP:conf/icml/ChooGL024} and counting the number of matches that differ from the optimal solution~\cite{DBLP:conf/nips/AntoniadisGKK20}, or for online caching, the absolute distance between real and predicted requests for each page~\cite{lykouris2021competitive} and the size of the xor between optimal and predicted cache contents~\cite{antoniadis2023online}).

This lack of a common ground poses multiple issues: it prevents adequate comparisons between results or transfers of algorithmic ideas, leaves room for analysis-driven error functions that artificially advantage certain algorithms or do not properly account for realistic errors, and prevents the emergence of assumption-free impossibility results on the smoothness metric.

\subparagraph*{Predictor probing.}

In part to address the limitations of error functions, \emph{quantitative} models were introduced, notably $\epsilon$-Accurate Predictions~\cite{gupta2022augmenting} and Infused Advice~\cite{EmekGP023}.
Those models replace the error function with a \emph{probability of failure} inspired by the success rate metric of ML predictors: each prediction is independently good or bad decided with coin tosses of a fixed probability.
While more systematic and both problem- and predictor-agnostic, these models share two restrictive assumptions.
First, bad predictions are random rather than adversarial%
\footnote{As assumed in~\cite{gupta2022augmenting} for the caching problem.}%
, which cannot model structured errors such as mislabeling or representation errors.
Second, the failure probability is fixed over the entire input sequence, which is unrealistic when predictor accuracy varies over time.
Together, these assumptions incentivize \emph{probing} strategies~\cite{gupta2022augmenting} that sample predictions to estimate the failure rate, and then exploit it --- a strategy that breaks down if accuracy changes mid-sequence.\\

The starting point of this paper is that existing learning-augmented algorithms rely on non-general assumptions about the predictor, restricting the applicability of their guarantees.
We therefore ask: \emph{does there exist a model whose formal guarantees depend solely on the problem?}

\subsection{The OAG Model}
We answer the aforementioned question in the affirmative by introducing a new model, called \emph{online algorithms with unreliable guidance (OAG)}, that stems directly from the formulation of online problems as \emph{request-answer games} \cite{BorodinEY1998book} and does not involve any ``external ingredients'' such as prediction semantics, error functions or probing strategies.
The OAG model brings three main innovations: a systematic manner to model predictions, a new manner to model the predictor's accuracy with the \emph{guide}, and a new manner to evaluate the algorithm's performance called the \emph{anytime-competitive ratio}; refer to \autoref{sec:model} for a formal definition of the OAG model.

\subparagraph*{Modeling predictions with guidance.}
We propose a universal standard for prediction semantics: predictions should take the form of answers the predictor would give if it were in the algorithm's place.
We call this prediction semantic \emph{guidance}.
Concretely, this means predictions are drawn directly from the answer space of the problem — a natural choice, since (almost) all useful online problems can be modeled as request-answer games~\cite{BorodinEY1998book}.

\subparagraph*{Modeling errors with a guide.}
Inspired by the success rate of an ML predictor, we model good and (adversarially) bad predictions with a \emph{guide}.
An OAG algorithm for an online problem $\P$ receives, with each incoming request $\rho_{t}$, a guidance $\gamma_{t}$ taken from the answer space of $\P$.
Ideally, $\gamma_{t}$ is a \emph{good guidance} that was generated knowing the entire request sequence and aiming to help the algorithm by selecting the optimal answer for $\rho_{t}$, however, $\gamma_{t}$ may also be a \emph{bad guidance} that selects the answer adversarially knowing the whole input sequence.
The choice between the two guidance at time $t$ is made by nature based on an independent $\beta$-biased coin toss, where
$0 \leq \beta \leq 1$
is a model \emph{bad guidance} parameter, so that $\gamma_{t}$ is bad with probability
$\beta$
and good with probability
$1 - \beta$.
Similarly to the approach of \cite{gupta2022augmenting} (where $\beta$ is analogous to
$1 - \epsilon$),
consistency and robustness are obtained by setting
$\beta = 0$
and
$\beta = 1$,
respectively, whereas smoothness arises naturally as $\beta$ shifts from $0$ to $1$.

\subparagraph*{Modeling predictor's changes with anytime competitiveness.}
We introduce \emph{anytime-competitiveness}, a metric that bounds performance against the offline optimum \emph{on every subsequence of the input sequence}.
Since guarantees hold for all subsequences and all values of $\beta$, algorithms have no incentive to probe the predictor: tuning the algorithm's behavior according to an estimated $\beta$ would only hurt performance on the current interval for other predictor accuracies.
This cleanly decouples the predictive and algorithmic parts by separating the acts of estimating the predictor's accuracy from the act of analyzing the algorithms' performance --- a third party may still probe and, upon detecting a change in accuracy, adjust the influence of predictions in real time (e.g., by changing the trust parameter $\tau$, see next \autoref{sub:intro-dtb}), with our framework directly supplying the relevant guarantees to take informed decisions.

\begin{figure}[h]
  
  \centering
  \begin{tikzpicture}
    \draw[->,thick] (0,0) -- (10,0) node[right] {time};

    \draw[thick,teal] (1,0) -- (3,0);
    \draw[teal,thick] (1,-0.12) -- (1,0.12);
    \draw[teal,thick] (3,-0.12) -- (3,0.12);
    \draw[teal,decorate,decoration={brace,amplitude=6pt,mirror}]
      (1,-0.15) -- (3,-0.15)
      node[midway, below=8pt] {$\beta = 0.1$};

    \draw[thick,teal] (3.5,0) -- (6,0);
    \draw[teal,thick] (3.5,-0.12) -- (3.5,0.12);
    \draw[teal,thick] (6,-0.12) -- (6,0.12);
    \draw[teal,decorate,decoration={brace,amplitude=6pt,mirror}]
      (3.5,-0.15) -- (6,-0.15)
      node[midway, below=8pt] {$\beta = 0.9$};

    \draw[thick,teal] (6.5,0) -- (8.5,0);
    \draw[teal,thick] (6.5,-0.12) -- (6.5,0.12);
    \draw[teal,thick] (8.5,-0.12) -- (8.5,0.12);
    \draw[teal,decorate,decoration={brace,amplitude=6pt,mirror}]
      (6.5,-0.15) -- (8.5,-0.15)
      node[midway, below=8pt] {$\beta = 0.4$};
  \end{tikzpicture}
  \caption{
    A request sequence along which the predictor's accuracy varies between request subsequences.
    Anytime competitiveness ensures that the algorithm has provable performance guarantees for each of those subsequences.
    A third-party prober that keeps accuracy estimates can detect changes in the predictor's behavior and use anytime competitive guarantees to take an informed decision in real time.
  }
\end{figure}

\subsection{Drop or Trust Blindly: Augmenting Any Algorithm with Predictions}
\label{sub:intro-dtb}
Beyond its elegant analytical framework, the OAG model opens the gate for a generic method, referred to as the \emph{drop or trust blindly (DTB) compiler}, that transforms, in a blackbox manner, any ``standard'' (a.k.a.\ \emph{prediction-free}) online algorithm into a learning-augmented algorithm in the OAG model.
In addition to the excellent performances of some of its output learning-augmented algorithms (see \autoref{sec:contribution}), the DTB compiler is the first proposal for a systematic manner to generate new learning-augmented algorithms to serve as a canonical reference point.

To see how the DTB compiler works, consider a prediction-free algorithm $\Alg$ and an incoming request $\rho_{t}$ and let $D_{t}$ be the distribution from which $\Alg$ picks the answer for $\rho_{t}$.
Assuming that the incoming guidance $\gamma_{t}$ belongs to the support of $D_{t}$, the OAG algorithm produced by the DTB compiler employs the following simple rule, where
$0 < \tau < 1$
is the compiler's \emph{trust parameter}:
with probability $\tau$, $\Alg^{*}$ adopts $\gamma_{t}$ blindly;
with probability
$1 - \tau$,
$\Alg^{*}$ picks the answer from $D_{t}$ (ignoring $\gamma_{t}$).
Refer to \autoref{sec:model} for a formal exposition of the DTB compiler.

On the face of it, the approach taken by the DTB compiler may be seen as ``too simple'', however, we prove that the resulting OAG algorithms actually provide attractive consistency-robustness guarantees for some classic and extensively studied online problems.

\subsection{Our Technical Contributions}
\label{sec:contribution}
In addition to our conceptual contributions, we develop OAG algorithms for classic online problems by applying the DTB compiler to the following well known prediction-free online algorithms:
the ranking algorithm of \cite{DBLP:conf/stoc/KarpVV90} for the \emph{bipartite matching} problem with adversarial arrival order (\autoref{sec:bipartite-matching});
the random marking algorithm of \cite{Fiat1991CompetitivePA} for the \emph{caching} problem (\autoref{sec:online-caching});
and
the algorithm of \cite{BorodinLS92} for the \emph{uniform metrical task system} problem (\autoref{sec:mts}).
We analyze those algorithms in the OAG model and anytime competitiveness guarantees expressed with the bad guidance parameter $\beta$ and the trust parameter $\tau$.
The consistency and robustness of our algorithms are obtained by setting
$\beta = 0$
and
$\beta = 1$,
respectively;
see \autoref{tab:consistency-robustness-guarantees}.
By modifying $\tau$, our algorithms can be made more risky or more prudent with respect to trusting the guidance they receive.

\begin{table*}
  \centering
  \begin{tabular}{c|c|c}
    \toprule
    online problem & consistency & robustness
    \\
    \hline
    bipartite matching &
                         $\frac{1 - e^{-(1 - \tau)}}{1 - \tau}$ &
                                                                  $\max \left\{ \frac{1}{2}, 1 - e^{-(1 - \tau)} \right\}$
    \\
    caching (cache size $k$) &
                               $\min \left\{ \frac{2}{\tau}, 2 H_{k} \right\}$ &
                                                                                 $\min \left\{ \frac{2 H_{k}}{1 - \tau}, k \right\}$
    \\
    uniform metrical task system ($n$ states) &
                                                $2 + 2 \cdot \min \left\{ \frac{1}{\tau}, (1 - \tau) H_{n} \right\}$ &
                                                                                                                       $2 + 2 \cdot \min \left\{ \frac{1}{1 - \tau} H_{n}, n - 1 \right\}$
    \\
    \bottomrule
  \end{tabular}
  \caption{\label{tab:consistency-robustness-guarantees}%
    The consistency and robustness guarantees of our OAG algorithms, expressed in terms of the trust parameter
    $0 < \tau < 1$.}
\end{table*}

Our results for online bipartite matching give the first non-trivial trade-off between consistency and robustness under adversarial arrival order~\cite{DBLP:conf/stoc/KarpVV90}.
Due to its many applications (most notably ads allocation), online bipartite matching has been intensively studied in the learning-augmented framework in multiple variants: random arrival order~\cite{DBLP:conf/icml/ChooGL024,DBLP:conf/nips/AntoniadisGKK20}, random graph~\cite{DBLP:conf/nips/AamandCI22}, and two-stage arrival~\cite{DBLP:conf/nips/JinM22}.
We are the first to derive non-trivial points on the consistency-robustness Pareto frontier for the original version of the problem.
For the online caching and uniform metrical task system problems, our algorithms (with constant
$0 < \tau < 1$)
admit asymptotically optimal constant consistency and logarithmic robustness, while being arguably simpler than the existing algorithms that have such guarantees (see, e.g.,\cite{lykouris2021competitive}).

\subsection{Model limitations}
We see two limitations of the OAG model.
First, it still assumes a probabilistic predictor (independent coin tosses per request) while the model was designed to avoid predictor-specific assumptions.
Second, it only captures quantitative errors and cannot express qualitative ones as error functions do.
Despite their drawbacks, error functions offer finer-grained error characterization --- particularly for ``one-shot'' problems (e.g., ski rental~\cite{KumarPS2018improving}) where quantitative approaches give deceptively simple guarantees; the OAG model is best suited for problems with long input sequences.

\section{Model}
\label{sec:model}

\subparagraph{Request-Answer Games.}
Consider an online minimization (resp., maximization) problem $\P$ defined over a \emph{request} space $\R$ and an \emph{answer} space $\A$.
An instance of $\P$ is given by a finite request sequence
$\rho \in \R^{*}$;
a solution for $\rho$ is an answer sequence
$\sigma \in \A^{*}$
satisfying
$|\sigma| = |\rho|$.
The quality of the solutions is measured by a \emph{cost function}
\mbox{$f_{\P} : \bigcup_{\ell \geq 0} \R^{\ell} \times \A^{\ell}
  \rightarrow
  \Reals_{ \geq 0} \cup \{ \infty \}$}
(resp., a \emph{payoff function}
\mbox{$f_{\P} : \bigcup_{\ell \geq 0} \R^{\ell} \times \A^{\ell}
  \rightarrow
  \Reals_{ \geq 0} \cup \{ -\infty \}$)}
that determines the cost (resp., payoff) 
$f_{\P}(\rho, \sigma)$
associated with applying the solution
$\sigma \in \A^{\ell}$
to the request sequence
$\rho \in \R^{\ell}$.

On an input request sequence
$\rho = (\rho_{1}, \dots, \rho_{\ell}) \in \R^{\ell}$,
a (randomized) \emph{online algorithm} $\Alg$ for $\P$ constructs the solution
$\sigma = (\sigma_{1}, \dots, \sigma_{\ell}) \in \A^{\ell}$
in an online fashion so that request $\rho_{t}$ is revealed to $\Alg$ at (discrete) time
$t = 1, \dots, \ell$
and in response, $\Alg$ (probabilistically) decides on the answer $\sigma_{t}$ irrevocably.
Formally, an online algorithm $\Alg$ is a family
\mbox{$\Alg = \left\{ \Alg_{t} \right\}_{t \geq 1}$}
of functions
\[
  \Alg_{t} : \R^{t - 1} \times \A^{t - 1} \times \R \times \{ 0, 1 \}^{\infty} \rightarrow \A
\]
that construct
the answer
\mbox{$\sigma_{t} = \Alg_{t} \left(
    (\rho_{1}, \dots, \rho_{t - 1}),
    (\sigma_{1}, \dots, \sigma_{t - 1}),
    \rho_{t},
    r
  \right)
  \in \A$}
at time $t$
based on
the request history
$(\rho_{1}, \dots, \rho_{t - 1}) \in \R^{t - 1}$,
the answer history
$(\sigma_{1}, \dots, \sigma_{t - 1}) \in \A^{t - 1}$,
the incoming request
$\rho_{t} \in \R$,
and the current random coin tosses
$r \in \{ 0, 1 \}^{\infty}$.\footnote{%
  For the sake of simplifying some of the discussions in the sequel, our formulation assumes (without loss of generality) that the online algorithm has no access to past coin tosses.}

\subparagraph{Anytime Competitiveness.}
The gold standard for evaluating the performance of online algorithms is \emph{competitive analysis} that compares the cost/payoff of the algorithm on the entire request sequence to that of an optimal offline algorithm.
In this paper, we enhance this classic notion and introduce the notion of \emph{anytime competitive analysis}, requiring that the competitiveness of the online algorithm holds for any time interval of the request sequence.

To this end, consider a request sequence
$\rho = (\rho_{1}, \dots, \rho_{\ell}) \in \R^{\ell}$
and times
$0 \leq t^{0} < t^{1} \leq \ell$.
Let
$S(\rho, t^{0}) \subseteq \A^{t^{0}}$
be the set of answer sequences
$\sigma^{0} = (\sigma^{0}_{1}, \dots, \sigma^{0}_{t^{0}})$
such that when $\Alg$ runs on $\rho$, it constructs $\sigma^{0}$ at times
$t = 1, \dots, t^{0}$
with a positive probability (that is,
$S(\rho, t^{0})$
is the support of the probability distribution from which the answer prefix of $\Alg$ is picked when $\Alg$ runs on $\rho$).

Fix some answer sequence
$\sigma^{0} = (\sigma^{0}_{1}, \dots, \sigma^{0}_{t^{0}}) \in S(\rho, t^{0})$
and let
$(\sigma_{t^{0} + 1}, \dots, \sigma_{t^{1}}) \in \A^{t^{1} - t^{0}}$
be the answer sequence constructed by $\Alg$ when it runs on $\rho$ at times
$t = t^{0} + 1, \dots, t^{1}$,
conditioned on the event that $\Alg$ constructs the request sequence $\sigma^{0}$ at times
$t = 1, \dots, t^{0}$;
notice that
$(\sigma_{t^{0} + 1}, \dots, \sigma_{t^{1}})$
is a random variable that depends on the coin tosses of $\Alg$ at times
$t = t^{0} + 1, \dots, t^{1}$.
Define the cost (resp., payoff) of $\Alg$ on $\rho$ during the time interval
$(t^{0}, t^{1}]$
given the answer prefix $\sigma^{0}$, denoted by
$\Alg(\rho, t^{0}, t^{1}, \sigma^{0})$,
as
\[
\Alg \left( \rho, t^{0}, t^{1}, \sigma^{0} \right)
\, = \,
\Ex \left( f_{\P} \left( \rho, \left( \sigma^{0}_{1}, \dots, \sigma^{0}_{t^{0}} \sigma_{t^{0} + 1}, \dots, \sigma_{t^{1}} \right) \right) \right)
-
f_{\P} \left( (\rho_{1}, \dots, \rho_{t^{0}}), \sigma^{0} \right)
\, .
\]

Let
$(\sigma^{*}_{t^{0} + 1}, \dots, \sigma^{*}_{t^{1}}) \in \A^{t^{1} - t^{0}}$
be a request sequence that minimizes (resp., maximizes)
$f_{\P} \left( \rho, (\sigma^{0}_{1}, \dots, \sigma^{0}_{t^{0}}, \sigma^{*}_{t^{0} + 1}, \dots, \sigma^{*}_{t^{1}}) \right)$,
that is, a request sequence constructed by an omnipotent optimal algorithm that knows the entire request sequence $\rho$ in advance (i.e., an optimal \emph{offline algorithm}), given the answer prefix $\sigma^{0}$.
Denote
\[
\Opt \left( \rho, t^{0}, t^{1}, \sigma^{0} \right)
\, = \,
f_{\P} \left( \rho, \left( \sigma^{0}_{1}, \dots, \sigma^{0}_{t^{0}} \sigma^{*}_{t^{0} + 1}, \dots, \sigma^{*}_{t^{1}} \right) \right)
-
f_{\P} \left( (\rho_{1}, \dots, \rho_{t^{0}}), \sigma^{0} \right)
\, .
\]

Online algorithm $\Alg$ is said to be \emph{anytime $c$-competitive} if there exists a constant $d$ such that for every
integer
$\ell \geq 0$,
request sequence
$\rho = (\rho_{1}, \dots, \rho_{\ell}) \in \R^{\ell}$,
times
$0 \leq t^{0} < t^{1} \leq \ell$,
and answer sequence
$\sigma^{0} = (\sigma^{0}_{1}, \dots, \sigma^{0}_{t^{0}}) \in S(\rho, t^{0})$,
it is guaranteed that
\mbox{$\Alg(\rho, t^{0}, t^{1}, \sigma^{0}) \leq c \cdot \Opt(\rho, t^{0}, t^{1}, \sigma^{0}) + d$}
(resp.,
\mbox{$\Alg(\rho, t^{0}, t^{1}, \sigma^{0}) \geq c \cdot \Opt(\rho, t^{0}, t^{1}, \sigma^{0}) - d$}).\footnote{%
Notice that the classic definition of competitiveness is obtained as a special case of anytime competitiveness by fixing
$t^{0} = 0$
and
$t^{1} = \ell$.}

\sloppy
\subparagraph{The OAG Model.}
An \emph{online algorithm with unreliable guidance (OAG)} $\Alg$ for $\P$ is an online algorithm augmented with a \emph{guidance sequence}
$\gamma = (\gamma_{1}, \dots, \gamma_{\ell}) \in \A^{\ell}$,
where $\ell$ is the length of the request sequence,
that ideally provides the optimal answers to $\Alg$, but should be regarded with caution as it is not fully trustworthy.
Formally, an OAG algorithm $\Alg$ is a family
\mbox{$\Alg = \left\{ \Alg_{t} \right\}_{t \geq 1}$}
of functions
\[
  \Alg_{t} : \R^{t - 1} \times \A^{t - 1} \times \R \times \A \times \{ 0, 1 \}^{\infty} \rightarrow \A
\]
that
construct the answer
\mbox{$\sigma_{t} = \Alg_{t} \left(
    (\rho_{1}, \dots, \rho_{t - 1}),
    (\sigma_{1}, \dots, \sigma_{t - 1}),
    \rho_{t},
    \gamma_{t},
    r
  \right)
  \in \A$}
at time $t$
based on
the request history
$(\rho_{1}, \dots, \rho_{t - 1}) \in \R^{t - 1}$,
the answer history
$(\sigma_{1}, \dots, \sigma_{t - 1}) \in \A^{t - 1}$,
the incoming request
$\rho_{t} \in \R$,
the incoming guidance
$\gamma^{t} \in \A$,
and the current random coin tosses
$r \in \{ 0, 1 \}^{\infty}$.
\par\fussy

Ideally, the guidance sequence $\gamma$ is constructed by a \emph{guide} $\G$ that aims to minimize $\Alg$'s cost (resp., maximize $\Alg$'s payoff).
Formally, the guide is a family
\mbox{$\G = \left\{ \G_{\ell, t} \right\}_{\ell \geq 1, 1 \leq t \leq \ell}$}
of functions
\[
  \G_{\ell, t} : \R^{\ell} \times \A^{t - 1} \rightarrow \A
\]
that generate a guidance
\mbox{$\gamma^{\mathsf{g}}_{t} = \G_{\ell, t} \left( \rho, (\sigma_{1}, \dots, \sigma_{t - 1}) \right) \in \A$}
at time $t$
based on
the entire request sequence
$\rho \in \R^{\ell}$
and
the answer history
$(\sigma_{1}, \dots, \sigma_{t - 1}) \in \A^{t - 1}$.
The crux of the OAG model is that at each time
$1 \leq t \leq \ell$,
the guidance generated by the guide is replaced with an adversarially corrupted guidance
$\gamma^{\mathsf{c}}_{t} \in \A$
independently with probability $\beta$, where
$0 \leq \beta \leq 1$
is a model \emph{bad guidance} parameter;
that is, the guidance $\gamma_{t}$ provided to $\Alg$ is set to
$\gamma^{\mathsf{g}}_{t}$
with probability
$1 - \beta$,
and to $\gamma^{\mathsf{c}}_{t}$ with probability $\beta$.

To make the model formulation precise,
given a function
$c : [0, 1] \rightarrow \Reals_{> 0}$,
an OAG algorithm $\Alg$ is said to be anytime $c(\beta)$-competitive if
there exists a guide $\G$ such that for every bad guidance parameter
$0 \leq \beta \leq 1$
and for every corrupted guidance sequence
$(\gamma^{\mathsf{c}}_{1}, \dots, \gamma^{\mathsf{c}}_{\ell}) \in \A^{\ell}$,
the algorithm is guaranteed to be anytime $c(\beta)$-competitive assuming that
$\gamma_{t} \gets \gamma^{\mathsf{g}}_{t}$
with probability
$1 - \beta$
and
$\gamma_{t} \gets \gamma^{\mathsf{c}}_{t}$
with probability
$\beta$
for each time
$1 \leq t \leq \ell$,
independently.
The function
$c = c(\beta)$
that optimizes the anytime competitiveness of $\Alg$ for each
$0 \leq \beta \leq 1$
is referred to as $\Alg$'s \emph{smoothness}, whereas $c(0)$ and $c(1)$ are referred to as $\Alg$'s \emph{consistency} and \emph{robustness}, respectively.

\sloppy
\subparagraph{The DTB Compiler.}
In this paper, we restrict our attention to a class of OAG algorithms that are derived from (prediction-free) online algorithms via a blackbox transformation.
This transformation, referred to as the \emph{drop or trust blindly (DTB) compiler}, is parameterized by a \emph{trust parameter}
$0 \leq \tau \leq 1$ (we may occasionally restrict to non-extreme values of $\tau$ for simplicity).
Given an online algorithm
$\Alg = \{ \Alg_{t} \}_{t \geq 1}$,
the OAG algorithm
$\Alg^{*} = \{ \Alg^{*}_{t} \}_{t \geq 1}$
obtained from $\Alg$ by applying the DTB compiler with trust parameter $\tau$ is defined as follows.
\par\fussy

\sloppy
Consider time
$t \geq 1$
and fix some
request history
$(\rho_{1}, \dots, \rho_{t - 1}) \in \R^{t - 1}$,
answer history
$(\sigma_{1}, \dots, \sigma_{t - 1}) \in \A^{t - 1}$,
and
incoming request
$\rho_{t} \in \R$.
Let
$D_{t} \in \Delta(\A)$
be the probability distribution from which
$\Alg_{t}((\rho_{1}, \dots, \rho_{t - 1}), (\sigma_{1}, \dots, \sigma_{t - 1}), \rho_{t}, r)$
is picked and let
$V_{t} = \Support(D_{t})$
be the set of answers in the support of $D_{t}$, referred to as \emph{valid answers}.\footnote{%
  Depending on $\Alg$, we may wish to change the description of the DTB compiler so that the valid action set $V_{t}$ includes all answers
  $\sigma \in \A$
  that are legal for the incoming request $\rho_{t}$ in the sense that they do not lead to $\infty$ cost (resp., $-\infty$ payoff).
  This is without loss of generality as $\Alg$ can be modified so that $\Support(D_{t})$ consists of all such legal answers while changing the overall cost (resp., payoff) by a negligible amount.}
Let $\gamma_{t}$ be the incoming guidance and let
$\varphi_{\text{trust}}$ be the outcome of a fresh Bernoulli trial with success probability $\tau$.
Denoting the answer of $\Alg^{*}$ at time $t$ by $\sigma_{t}$, the function
$\sigma_{t}
=
\Alg^{*}_{t}((\rho_{1}, \dots, \rho_{t - 1}), (\sigma_{1}, \dots, \sigma_{t - 1}), \rho_{t}, \gamma_{t}, r)$
is constructed
by setting
$\sigma_{t} = \gamma_{t}$
if
$\gamma_{t} \in V_{t}$
and
$\varphi_{\text{trust}} = \text{success}$;
and by picking $\sigma_{t}$ from $D_{t}$ otherwise (i.e., if
$\gamma_{t} \notin V_{t}$
or
$\varphi_{\text{trust}} = \text{failure}$).
\par\fussy



\section{Online Bipartite Matching}
\label{sec:bipartite-matching}

An instance of the online bipartite matching problem consists of a bipartite graph $G = (U, V, E)$ and a permutation $\pi$ of the nodes of $U$.
The nodes in $U$ arrive one by one in the order $\pi$, potentially adversarial.
Each time a node $u \in U$ arrives, its incident edges are revealed.
The online algorithm must then decide to match $u$ right away and in an irrevocable manner with one of its neighbors in $V$ that was not already matched.
The goal is to maximize the size of the obtained matching.

The online bipartite matching problem was first solved in 1990 by Karp, Vazirani and Vazirani~\cite{DBLP:conf/stoc/KarpVV90} who presented the well-known Ranking algorithm and proved it achieves a competitive ratio always greater than $1 - e^{-1}$ against an oblivious adversary.
They also showed that this competitive ratio is optimal up to lower-order factors in the number of nodes.
To the best of our knowledge, and despite multiple works on related problems~\cite{DBLP:conf/nips/AamandCI22,DBLP:conf/icml/ChooGL024,DBLP:conf/nips/AntoniadisGKK20,DBLP:conf/nips/JinM22}, there exists no learning-augmented algorithm for the online bipartite matching problem.

This section presents the DTB variant of the Ranking algorithm, \bipalg, whose pseudo-code can be found in appendix~(\autoref{alg:ranking-dtb}).
We derive an anytime-competitive ratio for $\bipalg$ which interpolates smoothly around the optimal, prediction-free competitive ratio of $1 - e^{-1}$ and depends on the values of $\beta$ and $\tau$:

\begin{restatable}{theorem}{thmBipartiteRanking}
  \label{thm:comp-ratio-bipartite-ranking}
  $\bipalg$ is $\max \left\{\frac{1}{2},  \; \big(1 - \beta \tau\big) \cdot \frac{1 - e^{- (1 - \tau)}}{1-\tau} \right\}$-competitive anytime in the OAG model.
\end{restatable}
\begin{proof}
  Proof in Appendix~(\autoref{sub:appendix-main-proof-bipartite}).
\end{proof}

From now on, we call $M^{*}$ a maximum matching (the optimal offline solution); let $x \in U \cup V$, we call $m^{*}(u) \in V$ the match of $x$ according to $M^{*}$.
To fix a guide, we assume that a good guidance suggests to match according to $M^{*}$.
When $u \in U$ is revealed, we call $N(u)$ the set of neighbors of $u$ that were not already matched by the considered algorithm.
The proof of~\autoref{thm:comp-ratio-bipartite-ranking} is inspired by the work of Birnbaum and Mathieu~\cite{DBLP:journals/sigact/BirnbaumM08}.




\subparagraph*{The perfect matching hypothesis.}

As in prior work~\cite{DBLP:conf/stoc/KarpVV90,DBLP:journals/sigact/BirnbaumM08}, we restrict to graphs admitting a \emph{perfect matching}: this ensures positive matching probabilities and fixes the optimal payoff at $n$ (for a graph with $2n$ nodes). In the prediction-free setting this is natural since Ranking's worst cases are perfect-matching graphs. This assumption is no more true in the OAG model --- bad guidance could in principle make other graphs harder for $\bipalg$---but our first claim~(\autoref{lem:node-removal}) shows that perfect-matching graphs remain worst cases provided that bad guidance plays optimally to minimize our anytime competitive ratio. We assume henceforth that $G$ admits a perfect matching, with $n = |V| = |U|$.

\begin{restatable}[Node Removal]{lemma}{lemNodeRemoval}
  \label{lem:node-removal}
  Fix a graph $G$, an arrival order $\pi$.
  Let $\mathcal{G}^{\text{bad}}$ be the bad function that minimizes the expected competitive ratio of $\bipalg$ in this context.
  Assume there exists a node $x \in U \cup V$ that is not matched by the maximum matching $M^{*}$ and define $G'$ the graph $G$ where $x$ was removed.
  Then there exists a bad guidance so that the expected competitive ratio of $\bipalg$ on $G'$ is no greater that the expected competitive ratio on $G$.
\end{restatable}
\begin{proof}
  Proof in Appendix~(\autoref{sub:appendix-proof-node-removal}).
\end{proof}

\subparagraph*{Independence under the guide's influence.}

The original Ranking analysis~\cite{DBLP:conf/stoc/KarpVV90} exploits correlations across decisions: since all nodes share the same ranking, the probability that a node is unmatched can be bounded via the expected number of nodes matched so far: matching few nodes increases the chances to match more nodes later. \autoref{lem:induction-lemma} adapts this argument to the OAG model.

Establishing this requires a key independence property in order to contain the guide's influence.
To see this, let $t \in \llbracket n \rrbracket$, let $v \in V$ be the (random) node of rank $t$.
Assume $v$ is not matched.
We call $u = m^*(v)$ the optimal match of $v$, and $R_{t-1} \subseteq U$ the set of nodes already matched to some $v' \in V$ of rank less than $t$.
Since by assumption the ranking algorithm (not the guide) matched $u$, we have $u \in R_{t-1}$---otherwise $u$ would have been matched to $v$.
Deriving~\autoref{lem:induction-lemma} would require $u$ and $R_{t-1}$ to be independent, which unfortunately fails in general~\cite{DBLP:journals/sigact/BirnbaumM08}.

Both~\cite{DBLP:conf/stoc/KarpVV90,DBLP:journals/sigact/BirnbaumM08} resolve this by replacing $u$ with a node drawn uniformly at random, which is independent of $R_{t-1}$.
This requires first verifying that permuting node ranks leaves the distribution of output matchings unchanged---which holds in the OAG model, since relabeling ranks carries the guide's behavior with it.
We introduce the following notation for \emph{replaced ranking}:

\begin{definition}[Replaced Ranking]
  \label{def:replaced-rank}
  Let $t \in \llbracket n \rrbracket$ and let $\sigma$ be a ranking of the nodes in $V$.
  We define $\sigma_{t,i}$ as the ranking $\sigma$ where the node with rank $t$ was removed and put back in with rank $i$.
\end{definition}

We then show the core of our proof: even under the guide's influence, transforming a ranking into a replaced ranking incurs at most a difference of one alternating chain between the two output matchings.

\subparagraph*{Unique alternating chain despite the guide.}

We want to show that, for any bad guidance, replacing a node in a ranking always has a small effect on the output matching; more precisely, executing $\bipalg$ on a replaced ranking outputs a matching that differs by at most one alternating chain.
We show that, despite the influence of the guide, a similar claim holds in the OAG model.
Intuitively we show that, (1) an alternating chain cannot start when $\bipalg$ uses the guidance to make its decisions --- the guide has no access to the algorithms' random seed --- and (2) if an alternating chain has already started in the past, we use the fact that both executions of $\bipalg$ (with and without replacement in the ranking) receive the same guidance, implying that the guide either does not interfere with the chain (if the suggested match is possible in both executions) or else makes the chain longer (if the suggested match is possible in only one execution, hence following the guidance implies to match with a node at the end of an alternating chain).

\sloppy
\begin{restatable}[Unique Alternating Chain]{lemma}{lemUniqueAlternatingChain}
  \label{lem:unique-alternating-chain}
  Let $t, i \in \llbracket n \rrbracket$.
  Let $\sigma$ be a ranking and let $M$ be the output matching of $\bipalg$ using ranking $\sigma$.
  We call $\bipalg_{t,i}$ the same algorithm as $\bipalg$ except it uses the replaced ranking $\sigma_{t,i}$, let $M_{t,i}$ be the output matching.
  Then assuming that $M$ and $M_{t,i}$ are not equal and that $M$ does not match some node $v \in V$, then $M$ and $M_{t,i}$ differ by one single alternating chain starting at $v$.
\end{restatable}
\par\fussy
\begin{proof}
  Proof in Appendix~(\autoref{sub:appendix-proof-unique-alternating-chain}).
\end{proof}

Once we established the uniqueness of the alternating chain, we can efficiently isolate the influence of the guide from our Ranking logic and deduce our independence result in~\autoref{lem:higher-rank}.
Intuitively this states that, provided that the ranking (not the guide) is used, the probability that a node $v$ with rank $t$ is not matched is upper-bounded by the probability that any node of $U$ ($i$ is free of choice, thus the desired independence property) is in $R_{t}$ using the replaced ranking $\sigma_{t,i}$.

\begin{restatable}[Higher Rank]{lemma}{lemHigherRank}
  \label{lem:higher-rank}
  Let $G$, $\pi$ be a problem instance, let $\gamma$ be a guidance sequence, let $\sigma$ be a ranking.
  Let $u \in U$ and let $v = m^{*}(u)$, let $t$ be the rank of $v$ in $\sigma$.
  Assuming that the ranking (not the guide) is used to match $u$ then, if $v$ is not matched under $\sigma$, for all $i \in \llbracket n \rrbracket$, $u$ is matched under $\sigma_{t,i}$ to a node with rank at most $t$ in $\sigma_{t,i}$.
\end{restatable}
\begin{proof}
  Proof in Appendix~(\autoref{sub:appendix-proof-higher-rank}).
\end{proof}

We finally derive our induction lemma which directly implies the main claim (\autoref{thm:comp-ratio-bipartite-ranking}) using an inductive reasoning.
The fact that an alternating chain is unique \emph{whatever the bad guidance} allowed us to decouple the losses in payoff generated by the guidance (left term $\beta \tau$) from those generated by the Ranking algorithm (right term $\frac{1 - \tau}{n} \cdot \sum_{1 \le s \le t} x_s$) using simple conditional probabilities.

\begin{restatable}[Induction Lemma]{lemma}{lemInductionLemma}
  \label{lem:induction-lemma}
  Let $t \in \llbracket n \rrbracket$.
  Given a problem instance and a prediction, let $x_t$ be the probability that the node of rank $t$ is matched.
  It holds:
  \begin{align*}
    1 - x_t \quad \le \quad \beta \tau \ + \ \frac{1 - \tau}{n} \cdot \sum_{1 \le s \le t} x_s
  \end{align*}
\end{restatable}
\begin{proof}
  Proof in Appendix~(\autoref{sub:appendix-proof-induction-lemma}).
\end{proof}

\section{Online Caching}
\label{sec:online-caching}

The online caching problem considers a set of pages and a memory (cache) which can store up to $k$ pages.
A sequence of pages (requests) arrives in an online manner.
Each time a page arrives, it must be \emph{cached}, that is, the page must be fetched into the cache in case it was not cached already.
An algorithm which solves the caching problem must maintain the cache over time so that each page is cached when it was requested and ensure that the cache never contains more than $k$ pages.
Fetching a page costs $1$ and the goal of the algorithm is to minimize the overall cost.

The online caching problem~\cite{Fiat1991CompetitivePA} is the most studied online problem in the learning-augmented framework.
The seminal paper from Lykouris and Vassilvitskii~\cite{lykouris2018} already focused on randomized caching and gave (asymptotic) optimal robustness and consistency bounds.
Most notable follow-up works solved more general variants such as the weighted case~\cite{bansal2022learning,jiang2022online} while others focused on improving the smoothness bound~\cite{rohatgi2020near,DBLP:conf/approx/Wei20}.

In this section, we present the DTB variant of the well-known Random Mark algorithm~\cite{Fiat1991CompetitivePA} to solve the online caching problem.
We call our algorithm $\markalg$ (\autoref{alg:random-mark-dtb}, pseudocode in appendix \ref{sub:appendix-pseudocode-randomark}) and show it achieves the (asymptotic) optimal trade-off between consistency and robustness ($\beta = 0$ and $\beta = 1$, respectively) while its competitive ratio smoothly degrades as the predictive errors become greater (\autoref{thm:rm-dtb-competitive}):

\begin{restatable}{theorem}{thmCaching}
  \label{thm:rm-dtb-competitive}
  $\markalg$ (\autoref{alg:random-mark-dtb}) is $\min \left\{\frac{2}{\tau (1-\beta)}, \frac{2 H_k}{1 - \tau \beta}, k\right\}$-competitive anytime in the OAG model with an additive constant of $2k$.
\end{restatable}
\begin{proof}
  Proof in Appendix~(\autoref{sub:appendix-proof-caching}).
\end{proof}

Efficiently solving the caching problem amounts to evicting the pages that will be requested the furthest in time.
Hence, we assume that a good guidance suggests to evict such an (unmarked) page; here guidance overlaps with literature standars~\cite{lykouris2021competitive}.

\subparagraph*{blame chains.}
Our proof works on \emph{blame chains}.
A blame chain is a sequence of pages $r(1) \dots r(n)$ that were evicted by \markalg such that, for all $i \in \llbracket n-1 \rrbracket$, $r(i+1)$ was evicted because $r(i)$ was requested.
In addition, all pages of a blame chain are pages that should not have been evicted, except the last one.
As a result, each page that is not the last incurs a non-optimal cost for \markalg.

Our proof reasons on the \emph{length of blame chains}.
For $i \in [0, n]$, let $X_i$ be the random variable equal to the remaining number of pages in the blame chain starting from page $r(n-i)$ (conditioned on the event that $r(n-i)$ is part of the blame chain).
We now state the core of our proof with~\autoref{lem:shorter-path} which states that, whatever its strategy, \emph{a bad guidance can only shorten the expected remaining length of a blame chain}:

\begin{restatable}{lemma}{lemShorterPath}
  \label{lem:shorter-path}
  Let $i, s \in [0, n]$ be two nodes such that $i$ and $s$ are successive nodes in the blame chain in case the bad guidance was used.
  It holds:
  \begin{align*}
    \mathbb{E}[X_{i}] \ge \mathbb{E}[X_{s}].
  \end{align*}
\end{restatable}
\begin{proof}
  Proof in Appendix~(\autoref{sub:appendix-proof-shorter-path}).
\end{proof}

This result directly enable us to contain the nuisance of the bad guidance.
We comply with the anytime competitive ratio by simply upper-bounding the costs of the first and last phases of the considered interval by $k$.
Using induction reasonings, we can now derive a result on the expected length of a blame chain and obtain our main claim:

\begin{restatable}{lemma}{lemBlameChainExpLength}
  \label{lem:blame-chain-exp-length}
  Assuming that $\beta \tau < 1$, a blame chain has an expected length of at most $H_k / (1 - \beta \tau)$.
\end{restatable}
\begin{proof}
  Proof in Appendix~(\autoref{sub:appendix-proof-blame-chain}).
\end{proof}

\section{Uniform Metrical Task Systems}
\label{sec:mts}
In metrical task systems (MTS), a metric space $(S,d)$ is given, where $S$ is a set of $n$ \emph{states} and $d:S^{2}\to \mathbb{R}_{\geq 0}$ is a metric \emph{distance} function. A \emph{task} is a function $r:S\to \mathbb{R}_{\geq 0}$ that associates each state $s\in S$ with a \emph{processing cost} $r(s)$. For a sequence $\mathcal{R}=r_{1},\dots, r_{m}$ of $m$ tasks, we define a \emph{solution} as a sequence $\mathcal{A}=s_{1},\dots, s_{m}$ of states $s_{i}\in S$, where the solution $\mathcal{A}$ is said to \emph{serve} the task $r_{i}$ at state $s_{i}$. We define the \emph{cost} of a solution $\mathcal{A}$ naturally as the sum between the total processing cost and the total transition cost incurred by $\mathcal{A}$, i.e., $\mathtt{cost}(\mathcal{A})=\sum_{i\in [m]}r_{i}(s_{i})+\sum_{i\in [m-1]}d(s_{i},s_{i+1})$. In the online problem, the tasks $r_i$ are given one by one in an online manner and a solution needs to serve task $r_{i}$ immediately as it arrives. 

In this section, we focus on MTS on the uniform metric, i.e., a metric where $d$ satisfies $d(s,s')=1$ for every pair of distinct states $s,s'\in S$. We will present an algorithm $\mtsalg$ that extends the classical online algorithm of Borodin, Linial, and Saks \cite{BorodinLS92} to the OAG model.

\sloppy
Before describing $\mtsalg$, we recall the notions of phases and saturation from \cite{BorodinLS92}. Towards that goal, we extend the definition of processing cost to \emph{continuous} time. Concretely, for each discrete time $i$, we naturally define the processing cost of task $r_{i}$ in a (continuous) time interval $[t,t']\subseteq [i,i+1]$ as $(t'-t)\cdot r_{i}(s)$ for each state $s\in S$. Consider an MTS instance $\mathcal{R}=r_{1},\dots, r_{m}$. We define a partition of the (continuous) interval $[1,m+1]$ into sub-intervals $[t_{0}=1,t_{1}],[t_{1},t_{2}],\dots, [t_{\ell-1},t_{\ell}=m+1]$ referred to as \emph{phases}. Each phase begins with all states being \emph{unsaturated}. Consider some phase $p=[t_{j},t_{j+1}]$. A state $s$ is said to be \emph{saturated} for $p$ at time $t\in p$ if the total processing cost associated with $s$ during the time interval $[t_{j},t]$ is at least $1$. The phase ends immediately after all states become saturated. Observe that upon the arrival of a task $r_{i}$ at (discrete) time $i$, the algorithm can determine which states will become saturated for the current phase by time $i+1$.
\par\fussy

\subparagraph{The algorithm.} Consider the task $r_{i}$ arriving at time $i$. Let $p$ be the current phase, and let $s$ be the current state of $\mtsalg$. We denote by $V_{i}$ the set of \emph{valid} states in time $i$ and initialize $V_{i}=\emptyset$. If $s$ does not become saturated for $p$ by time $i+1$, then $\mtsalg$ sets $V_{i}=\{s\}$ (i.e., $\mtsalg$ stays in state $s$). Otherwise, if $p$ does not end by time $i+1$, then $\mtsalg$ sets $V_{i}$ to be the set of all states that are not saturated at time $i+1$ (notice that it is guaranteed that $V_{i}\neq \emptyset$ since $p$ does not end by time $i+1$). Finally, if $p$ ends by time $i+1$, then $\mtsalg$ sets $V_{i}=\{s_{\min}\}$ where $s_{\min}$ is a state that minimizes the processing cost $r_{i}$.  
Let $\gamma_{i}$ be the guidance given to $\mtsalg$ at time $i$. If $\gamma_{i}\notin V_{i}$, then $\mtsalg$ simply chooses a state uniformly at random from $V_{i}$. Suppose now that $\gamma_{i}\notin V_{i}$. Then, $\mtsalg$ moves to $\gamma_{i}$ with probability $\tau$; and moves to a state chosen uniformly at random from $V_{i}$ otherwise. 

This completes the description of $\mtsalg$. We now analyze its performance. For a trust parameter $\tau$ and a bad guidance parameter $\beta$, we establish the following.

\sloppy
\begin{restatable}{theorem}{thmMTS}\label{theorem:mts}
    $\mtsalg$ is $2\cdot \min\{\frac{1}{\tau(1-\beta)}+1,\frac{1-\tau}{(1-\tau\beta)^{2}}H_{n}+1,n\}$-competitive anytime in the OAG model.
\end{restatable}
\par\fussy
\begin{proof}
  Proof in Appendix~(\autoref{sub:mts-proof}).
\end{proof}


\section{Conclusion}
This paper presents advances in our understanding of the usage of machine-learned predictions to reliably solve a given task.
We show that any online algorithms can be adapted in a systematic, canonical manner to integrate input predictions.
While general, this transformation naturally provides the usual desired properties of learning-augmented algorithms:
consistency, robustness, and smoothness.
We analyze this transformation on three well-studied problems and show they match, or even beat, state-of-the-art performances.
We envision that this framework will be used as a reference point in the future to evaluate learning-augmented algorithms.

\bibliography{references}

\newpage

\appendix

\section{Related Works}
\subsection{Model Comparison}
Beyond the aforementioned online algorithms with predictions model of \cite{lykouris2021competitive} and $\epsilon$-accurate predictions model of \cite{gupta2022augmenting}, the literature on online decision making includes various other models that also bear certain relations to our OAG model.
In particular, the popular \emph{online algorithms with advice} model, introduced by Emek et al.~\cite{EmekFKR2011advice} and B\"{o}ckenhauer et al.~\cite{BoeckenhauerKKKM2017advice} (see \cite{boyar2017online} for a survey), considers online algorithms that receive advice from a fully trustworthy oracle, aiming to minimize the size of the advice.
Angelopoulos et al.~\cite{angelopoulos2024online} introduced a generalization of the online algorithms with advice model in which the advice may come from an adversarial source, with the objective of optimizing (the equivalent of) the algorithm's consistency and robustness, while still aiming for small advice.
There are also various works on algorithms with ``noisy advice'', i.e., the advice is generated by a trusted oracle, however, it is perturbed by some random noise.
The problems studied within this framework include
maximum independent set \cite{BravermanDSW24},
max cut \cite{Cohen-Addadd0LP24},
searching in a tree \cite{BoczkowskiFKR21,BoczkowskiFKR25},
and
spectral clustering \cite{DBLP:journals/corr/abs-2511-17326}.

Slightly further away from the current work, the \emph{randomly infused advice} model of Emek et al.~\cite{EmekGP023} performs beyond worst-case analysis of online algorithms by ``infusing'' an advice generated by a (trusted) oracle into the online algorithm's random bits;
this is related to our OAG model as the advice is infused for each incoming request independently based on a (biased) random coin toss.
Elias et al.~\cite{EliasKMM2024learning} take an opposite approach from ours and integrate the predictor in the learning problem, so it is no longer a blackbox and can 
also learn from the input.
Finally, Bateni et al.~\cite{BateniDJW24} study a setting with queries that can be directed to either an expensive trusted oracle or to a cheaper untrusted oracle.

\subsection{Problem-Specific Related Work}

The online bipartite matching problem~\cite{DBLP:conf/stoc/KarpVV90} has been intensively studied by the learning-augmented community.
So far, the existing literature focused on problem variants such as the random arrival model, initiated in \cite{DBLP:conf/nips/AntoniadisGKK20}, and later improved in \cite{DBLP:conf/icml/ChooGL024}, the random graph model~\cite{DBLP:conf/nips/AamandCI22} or the two-stage arrival model~\cite{DBLP:conf/nips/JinM22}.
In this paper, we focus on the standard variant as introduced in~\cite{DBLP:conf/stoc/KarpVV90} where both the graph and the nodes' arrivals are adversarial.

The online caching problem~\cite{Fiat1991CompetitivePA} is probably the most studied online problem in the learning-augmented framework.
The seminal paper from Lykouris and Vassilvitskii~\cite{lykouris2018} already focused on randomized caching and gave (asymptotic) optimal robustness and consistency bounds.
Most notable follow-up works solved more general variants such as the weighted case~\cite{bansal2022learning,jiang2022online} while others focused on improving the smoothness bound~\cite{rohatgi2020near,DBLP:conf/approx/Wei20}.
While our consistency and robustness bounds match the state-of-the-art, our smoothness measure is however not comparable with those existing works.

The online metrical task system problem with uniform costs was first introduced and optimaly solved by~\cite{BorodinLS92}.
It later received the attention of the learning-augmented community in~\cite{DBLP:conf/aistats/ChristiansonSW23} who presented optimal robustness and consistency trade-offs for the problem in its general version.


\section{Proofs and Algorithms for Online Bipartite Matching}
\subsection{Algorithm's Pseudocode for Online Bipartite Matching}

\begin{algorithm}[h]
  \caption{$\bipalg$}
  \label{alg:ranking-dtb}
  \begin{algorithmic}
    \STATE {\bfseries Input:} the set $V$, trust parameter $\tau \in [0, 1]$
    \STATE $\sigma \gets$ a permutation on $V$ chosen uniformly at random
    \WHEN{$u \in U$ and $N(u) \subseteq V$ are revealed}
    \IF{$|N(u)| > 0$}
    \STATE $r \gets $ a random number in $[0, 1]$
    \STATE $g \gets $ the node the guide suggests to match with $u$
    \IF{$r \le \tau$ and $g \in N(u)$}
    \STATE Match $u$ with $g$
    \ELSE
    \STATE Match $u$ with $v \in N(u)$ with lowest rank in $\sigma$
    \ENDIF
    \ENDIF
    \ENDWHEN
  \end{algorithmic}
\end{algorithm}

\subsection{Proof of \autoref{lem:node-removal}}
\label{sub:appendix-proof-node-removal}

\lemNodeRemoval*
\begin{proof}
  Fix a ranking $\sigma$ and fix a guidance sequence $\gamma$ where the guide thinks the problem instance is $(G, \pi)$.
  We consider two runs: $\bipalg(\tau, G, \pi, \sigma)$ and $\bipalg(\tau, G', \pi, \sigma)$ under the exact same guidance sequence $\gamma$, call $M$ and $M'$ their respective output matchings.
  We will show that $M'$ has a size no greater than $M$.
  If $M$ and $M'$ have equal sizes then the subclaim holds.
  Otherwise, let $C$ be an alternating chain between $M$ and $M'$.
  We will now show that $x$ is an extremity of $C$.
  We define the node $u$ such that
  \begin{align}
    \label{cond:arrival-time}
    \text{``$u$ is the node in $U \cap C$ with earliest arrival time''}
  \end{align}
  We now distinguish between two cases and show that $x$ is an extremity of $C$ in both: either (1) $u$ is matched in both $M$ and $M'$ or (2) $u$ is matched in one of them and not in the other.\\
  We first consider case (1).
  Then call $v \in V$ and $v' \in V$ the nodes such that $\{u, v\} \in M$ and $\{u, v'\} \in M'$, it holds $v \neq v'$ by assumption.
  One of $v$ or $v'$ was not available to match for one of the matchings $M$ and $M'$.
  Assuming that both $v$ and $v'$ are nodes of $G'$ (i.e., $v \neq x$ and $v' \neq x$), we obtain that one of $v$ or $v'$ was already matched earlier, contradicting the assumption that $u$ is the node of $C \cap U$ with earliest arrival time.
  We deduce that $v = x$, proving the claim in case (1).\\
  We now consider case (2) where $u$ is matched in only one of $M$ or $M'$ but not in the other.
  Assuming that $u$ is a node in both $G$ and $G'$, then $N(u)$ was empty at the time $u$ was revealed for the matching that did not match $u$; calling $v \in V$ the node that is matched with $u$ in the other matching, $v$ was therefore matched earlier which contradicts the assumption that $u$ is the node with earliest arrival time in $C$.
  We deduce that $u$ is not in both $G$ and $G'$, therefore $u=x$ and $x$ is an extremity of $C$.\\
  As $x$ is an extremity of $C$, we deduce that $|M' \cap C| \le |M \cap C|$.
  Finally, we proved that $x$ is an extremity of any alternating chain which implies that there is at most one alternating chain, proving that $|M'| \le |M|$.

  We just showed that for a constant ranking $\sigma$ and constant guidance sequence $\gamma$, the output matching for $G'$ is no greater than for $G$.
  Hence for a constant guidance sequence $\gamma$, the expected competitive ratio for $G'$ is no greater than for $G$ --- notice the size of a maximum matching is the same for $G$ and $G'$.
  Finally notice that, going from $G$ to $G'$, the good guidance (which we assume blindly suggests to match like $M^{*}$) does not change and the bad guidance either does not change (as we assumed so far) or incurs a payoff even lower, the claim follows.
\end{proof}

\subsection{Proof of \autoref{lem:unique-alternating-chain}}
\label{sub:appendix-proof-unique-alternating-chain}

\lemUniqueAlternatingChain*
\begin{proof}
  Assume that $M$ and $M_{t,i}$ are not equal and that $M$ does not match $v$.
  We first show that the first alternating chain (chronologically) starts with $v$.
  Let $u$ be the first node of $U$ in the arrival order $\pi$ that is matched differently in $M$ and $M_{t,i}$.
  If the algorithms were following the guidance to match $u$ then they would both match $u$ to the same node as the matchings so far are the same --- note that the guide gave the same guidance for both algorithms since they do not have direct access to the rankings.
  The algorithms therefore do not follow the guidance at the time $u$ arrives, instead they both match according to their respective rankings.
  $\bipalg_{t,i}$ thus has $v$ as its candidate with lowest rank in $N(u)$ since it would match $u$ to the same node as $\bipalg$ otherwise: $M_{t,i}$ therefore matches $u$ with $v$.
  As $M$ does not match $v$, $v$ is at an end of an alternating chain, proving the subclaim.

  We then show that no more than one alternating chain can appear.
  Let $u \in U$ and assume that an alternating chain has already started when $u$ is revealed.
  In the case where both algorithm's executions use the ranking to match $u$, either both executions match $u$ to the same node, not interfering with the alternating chain, or $u$ is matched in two different ways (possibly matched or not matched) which implies that, in one execution, $u$ is matched to a node that is already part of the alternating chain--- except $v$, all nodes in $V$ have the same pairwise ordering in both rankings, thus only a node that is matched in one ranking but not in the other (hence, the end of an alternating chain) may incur different decisions. 
  Otherwise in the case where both algorithm's executions use the guidance to match $u$, either both executions match $u$ to the same node (as the guidance is the same for both executions), the guidance is impossible to follow in one of the executions (same case as before: the other execution matches $u$ with a node part of the alternating chain) or, finally, the guidance is impossible to follow in both executions, both falling back to the ranking (handled by the previous case).
\end{proof}

\subsection{Proof of \autoref{lem:higher-rank}}
\label{sub:appendix-proof-higher-rank}

\lemHigherRank*
\begin{proof}
  Let $M$ be the matching output under $\sigma$, let $M_{t,i}$ be the matching output under $\sigma_{t,i}$.
  First we show that, in the output matching under $\sigma$, $u$ is matched to a node with rank lower than $t$.
  $u$ cannot be matched to a node with rank $t$ as it is $v$ which we assumed is not part of the output matching, otherwise $u$ cannot be matched to a node with rank greater than $t$ as $v$ is available, our first subclaim holds.\\
  Then, we show the desired claim.
  Let $u_1, u_2 \dots u_l$ be the nodes in $U$ in reveal order that are matched (or not matched) in different manners under $\sigma$ and $\sigma_{t,i}$.
  If $u$ is not listed in $u_1 \dots u_l$ then the wanted claim directly holds: $u$ is matched to a node with rank lower than $t$ in $\sigma$ and the rank of that node increases by at most one in $\sigma_{t,i}$.
  In the remainder, we therefore assume that $u$ is listed in $u_1 \dots u_l$.
  We show that for each $i \in [1, l]$, at the reveal time of $u_i$, the set of matching candidates for $u_i$ under $\sigma_{t,i}$ strictly contains the set of matching candidates minus $v$ for $u_i$ under $\sigma$.
  We start with the case $i=1$.
  When $u_1$ is revealed, the current matching is the same under $\sigma$ and $\sigma_{t,i}$ and the set of matching candidates for $u_1$ is the same in both.
  $v$ is a neighbor of $u_1$ otherwise $u_1$ would be matched similarly under both rankings and the case $i=1$ holds.\\
  We now deal with the case $i \ge 2$.
  As stated in \autoref{lem:unique-alternating-chain}, $u_1 \dots u_{i-1}$ are the nodes of $U$ making up the unique alternating chain when $u_i$ is revealed; that chain starts with an edge of $M_{t,i}$ and ends with an edge of $M$ linked to $u_{i-1}$.
  All the internal nodes of that chain are both part of $M$ and $M_{t,i}$.
  For $u_i$, the beginning of the chain is a candidate node under $\sigma$ and the end of the chain is a candidate option under $\sigma_{t,i}$.
  Removing $v$ from the candidate set under $\sigma$ hence gives the wanted subclaim.\\
  We just proved that, removing the option to match $v$ under $\sigma$, all the nodes of $U$ that are part of the unique alternating chain (assuming it exists) have a set of matching candidates strictly greater under $\sigma_{t,i}$ than under $\sigma$ (in the sense of inclusion).
  Hence, assuming that $v$ is not matched under $\sigma$, $u$ is matched under $\sigma_{t,i}$ to a node with rank at most $t-1$ in $\sigma$ and so at most $t$ in $\sigma_{t,i}$.
\end{proof}

\subsection{Proof of \autoref{lem:induction-lemma}}
\label{sub:appendix-proof-induction-lemma}

\lemInductionLemma*
\begin{proof}
  Let $v \in V$ be the random variable equal to the node with rank $t$, let $u = m^*(v)$.
  Using total probabilities, we distinguish between three cases: when $\bipalg$ decides how to match $u$, it uses either the good guidance, the bad guidance or the ranking.
  If the good guidance is used, $v$ is therefore matched with probability one --- either $\bipalg$ matches $u$ with $v$, or $v$ was already matched before.
  In case the bad guidance happens when $u$ is matched, we simply upper-bound the probability that $v$ is matched by $1$, hence the $\beta \tau$ term on the right member of our desired claim.\\
  The rest of the proof seeks an upper-bound on the probability that $v$ is matched in case the random ranking is used when $u$ is matched.
  We consider the following alternative process: pick a node $v^{\prime} \in V$ uniformly at random, remove it from the ranking $\sigma$ and put it back in with rank $t$, call $\sigma^{\prime}$ the newly obtained ranking.
  We call $u^{\prime} = m^{*}(v')$ and $\bipalg^{\prime}$ the algorithm running with ranking $\sigma^{\prime}$.
  Clearly, the probability that $v^{\prime}$ is not matched equals $1 - x_t$.
  By \autoref{lem:higher-rank}, the event that $v^{\prime}$ is not matched under ranking $\sigma^{\prime}$ implies that $u^{\prime}$ is matched under ranking $\sigma$ with a node of rank no greater than $t$.
  Notice that $u^{\prime}$ is independent of $\sigma$, so the probability that $u^{\prime}$ is matched with a node of rank no greater than $t$ is $\frac{1}{n} \cdot \sum_{1 \le s \le t} x_s$, giving the desired result.
\end{proof}

\subsection{Proof of \autoref{thm:comp-ratio-bipartite-ranking}}
\label{sub:appendix-main-proof-bipartite}

\thmBipartiteRanking*
\begin{proof}
  First, notice that $\bipalg$ always outputs a matching that is maximal (i.e., it cannot be augmented greedily) which guarantees a competitive ratio greater than $1/2$.
  The rest of the proof shows that the competitive ratio is no lower than $(1-\beta\tau) \cdot (1 - e^{-(1-\tau)}) / (1-\tau)$.\\
  For all $i \in \llbracket n \rrbracket$ we define $S_i = \sum_{1 \le s \le i} x_i$.
  Like in \autoref{lem:induction-lemma}, we define $x_t$ as the probability the node with rank $t$ in $\sigma$ is matched (note that $\sigma$ is a random variable).
  To start, we prove by induction that, for all $i \in \llbracket n \rrbracket$ it holds:
  \begin{align*}
    S_t \ge (1 - \beta\tau) \cdot \sum_{s = 1}^{t} \left(1 - \frac{1-\tau}{n+1-\tau}\right)^{s}
  \end{align*}
  The base case for $t=1$ is obtained simply by rearranging the terms of \autoref{lem:induction-lemma} for $t=1$.
  By induction, we assume the result holds for $t$ and show it for $t+1$.
  Adding $S_{t+1}$ on both sides of \autoref{lem:induction-lemma} for $t+1$:
  \begin{align*}
    &1 + S_{t} \le \left(1 + \frac{1-\tau}{n}\right) \cdot S_{t+1} + \beta \tau\\
    &\implies S_{t+1} \ge  \left(1  - \frac{1-\tau}{n+1-\tau}\right) \cdot \big[ (1 - \beta\tau) + S_t\big]\\
    &\ge  (1  - \beta \tau) \cdot \sum_{s = 1}^{t+1} \left(1 - \frac{1-\tau}{n+1-\tau}\right)^{s}
  \end{align*}
  which proves the subclaim by induction. We therefore have that
  \begin{align*}
    \sum_{s=1}^{n} x_s \ = \ S_n &\ge (1 - \beta \tau) \cdot \sum_{s = 1}^{n} \left(1 - \frac{1-\tau}{n+1-\tau}\right)^{s}\\
                                 &\ge (1 - \beta \tau) \cdot \sum_{s = 1}^{n} \left(1 - \frac{1-\tau}{n}\right)^{s}\\
                                 &= (1 - \beta\tau) \cdot n \cdot \frac{1 - (1 - (1-\tau)/n)^{n+1}}{1-\tau}\\
                                 &\ge (1 - \beta\tau) \cdot n \cdot \frac{1 - e^{-(1-\tau)}}{1-\tau}
  \end{align*}
  Noticing that $\sum_{s=1}^{n} x_s$ is the expected profit of $\bipalg$ and that an optimal algorithm has profit $n$ ends the proof.
\end{proof}

\section{Proofs and Algorithm for Online Caching}

\subsection{Pseudocode of \markalg}
\label{sub:appendix-pseudocode-randomark}

\begin{algorithm}[t]
  \caption{$\markalg$}
  \label{alg:random-mark-dtb}
  \begin{algorithmic}
    \STATE {\bfseries Input:} cache size $k$, initial cache content, trust parameter $\tau \in [0, 1]$
    \STATE At the beginning, no pages are marked
    \WHEN{a request to a page $x$ arrives}
    \IF{the requested page is not cached}
      \IF{the cache is full}
        \STATE $r \gets $ a random number in $[0, 1]$
        \STATE $g \gets $ the page the guide suggests to evict
      \IF{$r \le \tau$ and $g$ is cached and non-marked}
      \STATE Evict $g$
      \ELSE
      \STATE Evict a non-marked page uniformly at random
    \ENDIF
    \ENDIF
    \ENDIF
    \STATE Mark $x$
    \IF{there are $k$ marked pages}
    \STATE Unmark all pages
    \ENDIF
    \ENDWHEN
  \end{algorithmic}
\end{algorithm}

\subsection{Proof of \autoref{lem:shorter-path}}
\label{sub:appendix-proof-shorter-path}

\lemShorterPath*
\begin{proof}
    We prove this by strong induction on $i = 0, 1 \dots n$.
    The claim holds for $i=0$.
    Let $i \in [0, n]$, we assume that the induction hypothesis holds for all $0, 1 \dots i-1$ and show it also holds for $i$.
    Let $r(n+1-s)$ be the next page after $r(n+1-i)$ on the blame chain if the bad guidance was used.
    It holds:
    \begin{align*}
      \mathbb{E}[X_{i}] &= 1 + \beta \tau \mathbb{E}[X_{s}] + \frac{1-\tau}{i} \cdot \sum\limits_{l=1}^{i-1} \mathbb{E}[X_{l}].
    \end{align*}
    Using the same equality for $s$ and then the induction hypothesis we obtain:
    \begin{align*}
      \mathbb{E}[X_{s}] &\le 1 + \beta \tau \mathbb{E}[X_{s}] + \frac{1-\tau}{s} \cdot \sum\limits_{l=1}^{s-1} \mathbb{E}[X_{l}]
    \end{align*}
    Moreover it holds:
    \begin{align*}
      \frac{1}{i} \cdot \sum\limits_{l=1}^{i-1} \mathbb{E}[X_{l}] \ge \frac{1}{s} \cdot \sum\limits_{l=1}^{s-1} \mathbb{E}[X_{l}]
    \end{align*}
    as the averaged terms on the left sum are either contained in the right sum or greater (induction hypothesis).
    Gathering the three previous inequalities gives the desired claim.
  \end{proof}

\subsection{Proof of \autoref{lem:blame-chain-exp-length}}
\label{sub:appendix-proof-blame-chain}

\lemBlameChainExpLength*
\begin{proof}
  Consider a blame chain $r(0), r(i_1) \dots r(n+1)$.
  Let $p \in [0, n]$ and let $r(s)$ be the next page after $r(p)$ in the blame chain in case the bad guidance was used.
  It holds:
  \begin{align*}
    \mathbb{E}[X_p] &= 1 + \beta \tau \mathbb{E}[X_{s}] + \frac{1 - \tau}{p} \sum\limits_{i=1}^{p-1} \mathbb{E}[X_{i}].
  \end{align*}
  Using \autoref{lem:shorter-path} gives:
  \begin{align*}
    (1 - \beta \tau) \cdot \mathbb{E}[X_p] &\le 1 + \frac{1 - \beta \tau}{p} \sum\limits_{i=1}^{p-1} \mathbb{E}[X_{i}]
  \end{align*}
  We now assume by strong induction that $\mathbb{E}[X_{i}] \le H_{i}/(1 - \beta \tau)$ for all $i \in [1, p-1]$ --- the base case for $i=1$ clearly holds.
  Dividing the previous inequality by $(1 - \beta \tau)$:
  \begin{align*}
    \mathbb{E}[X_p] &\le \frac{1}{1 - \beta \tau} + \frac{1}{p} \sum\limits_{i=1}^{p-1} \frac{H_{i}}{1 - \beta \tau}\\
    \implies \mathbb{E}[X_p] &\le \frac{1}{1 - \beta \tau} + \frac{1}{p \cdot (1 - \beta \tau)} \sum\limits_{i=1}^{p-1} \frac{p-i}{i}\\
    \implies \mathbb{E}[X_p] &\le \frac{1}{1 - \beta \tau} + \frac{1}{1 - \beta \tau} \cdot H_{p-1} - \frac{p-1}{p} \cdot \frac{1}{1 - \beta \tau}\\
    \implies \mathbb{E}[X_p] &\le \frac{H_{p}}{1 - \beta \tau}
  \end{align*}
  It therefore holds that $\mathbb{E}[X_n] \le H_n/(1 - \beta \tau) \le H_k/(1 - \beta \tau)$ and the claim follows.
\end{proof}

\subsection{Proof of \autoref{thm:rm-dtb-competitive}}
\label{sub:appendix-proof-caching}

\thmCaching*
\begin{proof}
  We assume that a good guidance suggests to evict the unmarked, cached page that will be requested latest.
  Let $\sigma$ be the sequence of requested pages.
  We define a \emph{phase} as a subsequence of consecutive requested pages that contains no more than $k$ different pages.
  We greedily divide $\sigma$ into phases: starting from the first request, take the longest possible phase, and so on until the end of $\sigma$.
  Let $P$ be a phase, we define:
  \begin{itemize}
  \item
    A \emph{clean} page is a page that is requested during $P$ but was not cached at the beginning of $P$.
    Define $C$ the number of clean pages.
  \item
    A \emph{returning} page is a page that is requested during $P$ and was cached at the beginning of $P$.
    Define $R$ the number of returning pages.
  \item
    A \emph{vanishing} page is a page that is not requested during $P$ and was cached at the beginning of $P$.
    There are $C$ vanishing pages.
  \end{itemize}

  \noindent
  First, notice that $\markalg$ pays no more than $k$ during $P$, since it pays at most one for each page kind thanks to the marking mechanism.

  \noindent
  Then, we prove that the expected number of page faults is upper bounded by $C / (\tau (1-\beta))$.
  With probability $\tau (1-\beta)$, the algorithm uses a good guidance and $\markalg$ evicts a vanishing page.
  At any given time of $P$, the number of evictions until the next vanishing page is evicted follows a geometric law of parameter $\tau (1-\beta)$.
  As there are $C$ vanishing pages, the expected number of evictions until all vanishing pages are out of the cache is $C / (\tau (1-\beta))$, which is an upper bound on the expected cost of $\markalg$ during phase $P$.

  \noindent
  Finally, there are $C$ blame chains evolving in parallel, each with an expected length of at most $H_k / (1 - \beta \tau)$ by \autoref{lem:blame-chain-exp-length}.
  The expected number of page faults is therefore no greater than $C \cdot H_k / (1 - \beta \tau)$.
  Meanwhile, an optimal offline algorithm pays an amortized cost of at least $C / 2$ per phase (well-known proof from \cite{Fiat1991CompetitivePA}) and the claim holds.
\end{proof}

\section{Proofs for Online Metrical Task System}

\subsection{Proof of \autoref{theorem:mts}}
\label{sub:mts-proof}

\thmMTS*
\begin{proof}
    Consider an arbitrary time interval $[t,t']$ and recall that our goal is to show that $\mtsalg$ is $2\cdot \min\{\frac{1}{\tau(1-\beta)}+1,\frac{1-\tau}{(1-\tau\beta)^{2}}H_{n}+1,n\}$-competitive in this interval. Let $p$ and $p'$ be the phases in which $t$ and $t'$ lie, respectively. Observe that the cost paid by $\mtsalg$ between time $t$ and the end of $p$ as well as the cost between the beginning of $p'$ and $t'$ can be attributed to the additive constant. Thus, to complete our analysis, it is left to analyze the competitive ratio in a single (complete) phase.    
    
    We first observe that any offline algorithm  must pay a cost of at least $1$ on every phase. Indeed, if an algorithm transitions during a phase, then it pays at least $1$ in transition cost. Otherwise, it resided in the same state throughout the phase. By definition, a phase continues until all states become saturated. Thus, any offline algorithm must pay at least $1$ in processing cost.

    It is therefore sufficient to provide an upper bound on the expected cost of $\mtsalg$ in a single phase $p$. Notice that by design, $\mtsalg$ incurs a cost of at most $1$ in processing cost for every state it visits during $p$. This is because $\mtsalg$ stays in a state only until it becomes saturated. In particular, $\mtsalg$ incurs a total processing cost of $2$ on the first and last states it visits during $p$. For each intermediate state, $\mtsalg$ incurs $1$ in transition cost and at most $1$ in processing cost. Therefore, to complete the analysis it is left to show that the expected number of \emph{transitions} during $p$ is bounded by $\min\{\frac{1}{\tau(1-\beta)},\frac{1-\tau}{(1-\tau\beta)^{2}}H_{n},n-1\}$.

    We start by showing that the number of transitions during $p$ is at most $n-1$. To see that, notice that by construction, whenever $\mtsalg$ transitions, it moves to an unsaturated state for $p$. Thus, in any transition, the number of unsaturated states decreases. It follows that the number of transitions during $p$ is at most $n-1$. 

\sloppy
    We now show that the expected number of transitions $\mtsalg$ makes during $p$ is upper bounded by $1/(\tau(1-\beta))$. To that end, we consider a guide that selects a state $s^{*}$ with longest time until saturation. Observe that if $\mtsalg$ transitions to $s^{*}$ during $p$, its next transition is only in the next phase. Recall that the probability of receiving non-corrupted guidance at each time $i$ is $\tau(1-\beta)$. Thus, the expected number of transitions during $p$ is at most $1/(\tau(1-\beta))$.
\par\fussy

    Finally, we need to show that the expected number of transitions in a phase is upper bounded by $\frac{1-\tau}{(1-\tau\beta)^{2}}H_{n}$. We say that a transition round is \emph{adversarial} if a corrupted guidance is given in that round. Let $X_{1}$ be a random variable that counts the number of adversarial rounds until the first non-adversarial round. For each $i>1$, let $X_{i}$ be a random variable that counts the number of adversarial rounds between the $(i-1)$-th and the $i$-th non-adversarial rounds. Let $N$ be a random variable that counts the total number of non-adversarial rounds in the phase.
    
    The goal is now to bound the sum $X_{1}+\dots+X_{N}$ which serves as an upper bound on the number of transitions during the phase. To that end, we note that we can assume w.l.o.g.\ that $N$ is determined strictly by the randomness of the algorithm. This is because non-adversarial rounds in which a non-corrupted guidance is given can only decrease the value of $N$ (and thus, the sum). Based on this assumption, we get that $N$ and the $X_{i}$ variables are determined based on two independent sources of randomness: the former by the algorithm's randomness, whereas the latter depends on the random assignment of bad guidance. In particular, $N$ is independent of the random variables $X_{i}$. Thus, it follows from Wald's identity that $\mathbb{E}[X_{1}+\dots+X_{N}]=\mathbb{E}[X_{1}]\cdot \mathbb{E}[N]$ (here, notice that $\mathbb{E}[X_{1}]=\mathbb{E}[X_{i}]$ for all $i$ since the $X_{i}$s are identically distributed). Since $X_{1}$ is a geometric variable with success probability $1-\tau\beta$, we get $\mathbb{E}[X_{1}]\cdot \mathbb{E}[N]=(1/(1-\tau\beta))\mathbb{E}[N]$.  
    
    It is left to show that $\mathbb{E}[N]\leq \frac{(1-\tau)}{1-\beta\tau}H_{n}$. Denote by $\nu_{k}$ the number of non-adversarial rounds in the phase given that there are $k$ unsaturated states and notice that $N=\nu_{n}$. Clearly, $\nu_{1}=1$. For $k>1$, first note that given that the round is non-adversarial, the probability of moving from $k$ unsaturated states to a state that leaves $k-1$ unsaturated states (i.e., the state with nearest saturation time) is at most $(1/k)(1-\tau)$. Now, conditioning on the round being non-adversarial we get a probability bound of $\frac{(1-\tau)}{1-\beta\tau}(1/k)$. Based on that, we get the recursive formula $\nu_{k}\leq \nu_{k-1}+\frac{(1-\tau)}{1-\beta\tau}(1/k)$. Developing the recurrence for $n$, we get $\nu_{n}\leq \frac{(1-\tau)}{1-\beta\tau}H_{n}$ which completes the analysis.
\end{proof}

\end{document}